\documentclass[journal]{IEEEtran}
\usepackage{times}
\usepackage{soul}
\usepackage{url}
\usepackage[hidelinks]{hyperref}
\usepackage[utf8]{inputenc}
\usepackage[small]{caption}
\usepackage{graphicx}
\usepackage{amsmath}
\usepackage{amsthm}
\usepackage{booktabs}
\usepackage[ruled,vlined,linesnumbered]{algorithm2e}
\usepackage[switch]{lineno}
\newtheorem{example}{Example}
\newtheorem{theorem}{Theorem}
\newtheorem{definition}{Definition}
\usepackage{amssymb}

\ifCLASSINFOpdf
\else
\fi
\begin{document}
%
\title{A Boolean Function-Theoretic Framework for Expressivity in GNNs with Applications to Fair Graph Mining}
%
%
%

\author{Manjish Pal,~\IEEEmembership{Member,~IEEE}
\thanks{M. Pal is with the National Law University Meghalaya, Shillong-793014 India.}
}

%
%

\markboth{IEEE Transactions on Computational Social Systems}%
{Shell \MakeLowercase{\textit{et al.}}: Bare Demo of IEEEtran.cls for IEEE Journals}
%



\maketitle

\begin{abstract}
We propose a novel expressivity framework for Graph Neural Networks (GNNs) grounded in Boolean function theory, enabling a fine-grained analysis of their ability to capture complex subpopulation structures. We introduce the notion of \textit{Subpopulation Boolean Isomorphism} (SBI) as an invariant that strictly subsumes existing expressivity measures such as Weisfeiler-Lehman (WL), biconnectivity-based, and homomorphism-based frameworks. Our theoretical results identify Fourier degree, circuit class (AC$^0$, NC$^1$), and influence as key barriers to expressivity in fairness-aware GNNs. We design a circuit-traversal-based fairness algorithm capable of handling subpopulations defined by high-complexity Boolean functions, such as parity, which break existing baselines. Experiments on real-world graphs show that our method achieves low fairness gaps across intersectional groups where state-of-the-art methods fail, providing the first principled treatment of GNN expressivity tailored to fairness.
\end{abstract}


%
\IEEEpeerreviewmaketitle

\section{Introduction}
Graph Neural Networks (GNNs) \cite{kipf2016semi} have emerged as a powerful framework for learning over graph-structured data. GNNs have been extensively explored in a wide range of applications including computer vision, natural language processing, molecular analysis,
data mining, financial risk control, bioinformatics  and anomaly detection. There also has been a growing body of work on using GNNs to implement fairness in downstream tasks like node classification and link prediction in graphs \cite{dong2022fairness}. Due to these diverse applications, significant research has been done to understand the limits to the computational and expressive power of GNNs \cite{zhang2024expressive}. The expressive power of GNNs refers to their ability to distinguish different graph structures or node configurations based on their representations. Formally, it characterises the ability of a GNN model to map structurally distinct graphs into distinct embeddings or predictions. A GNN with higher expressive power can capture nuanced graph features, effectively differentiating even subtle variations in topology or node attributes. In a pioneering study, \cite{morris2019weisfeiler, morris2020weisfeiler,xu2018how} introduced to use graph isomorphism recognition to explain the expressive power of GNNs, leading the trend to analyze the separation ability of GNNs. One of the key findings of the study is that the expressive power of certain GNNs is bounded by the 1-WL (Weissfeiler-Lehman) test \cite{weisfeiler1968reduction}, which is a well-known heuristic for the graph isomorphism problem i.e. two graphs $G$ and $H$ which are indistinguishable by the 1-WL test will produce same node representations by GNNs like GCN, GraphSage, GIN etc. In \cite{morris2019weisfeiler} the notion of $k$-GNN was introduced, whose expressive power is bounded by the $k$-WL test. Subsequently, a large body of work has developed that relies on graph structural properties (which also includes \emph{graph homomorphism} \cite{zhang2024beyond} and \emph{graph polynomials} \cite{puny2023equivariant} ) to define the expressive power of GNNs. Although notions of expressive power based on structural properties led to the development of more structurally expressive GNNs, none of the current techniques addresses the expressive power of GNNs in terms of both \emph{node features and structural properties} of the graph. Although there are GNN phenomena like \emph{oversmoothing} \cite{li2018deeper,oono2019graph} and \emph{oversquashing} \cite{topping2021understanding,di2023over} which rely on the dimensionality of node features, they don't take into account the node feature values. In addition, \cite{kanatsoulis2024graph} demonstrated that GNNs can distinguish any graph with at least one uniquely featured node, showing the importance of node features and that feature embeddings enhance structural representation. When node features sufficiently capture node constants, it has been shown \cite{loukas2019graph,vignac2020building} that GNNs with strong message-passing and update functions can be Turing universal. In such cases, structure plays a minimal role in node differentiation, as further illustrated by a simple graph isomorphism example \cite{loukas2020hard}.   
However, their limitations in expressivity as characterised by the Weisfeiler-Lehman (WL) hierarchy have raised concerns about their ability to capture global structural properties. This becomes particularly problematic in fairness-sensitive applications, where equitable treatment across diverse subpopulations of nodes is required.\\
In this work, we take a step forward to understand GNN expressiveness both in terms of node features and structural properties and  
formulate a fairness-aware perspective on GNN expressivity. Specifically,
given a graph $G$ and certain sensitive subpopulations $\cal G$, which may be defined from a set of sensitive attributes, we measure the expressive power of a GNN $\cal A$ in terms of $G$ and ${\cal G}$. To that end, we model each subpopulation as a subset of nodes and represent it via a binary indicator vector. We then define a Boolean function $f_{\cal G}$ that encodes, $\cal G$ and the output of a GNN is assumed to be dependent on both $f_{\cal G}$ and $G$. In particular, a GNN ensuring fairness across ${\cal G}$ is assumed to be crucially using $f_{\cal G}$. 

\section{Related Work and Background}
\subsection{Expressive Power of GNNs.} The expressive power of Graph Neural Networks (GNNs) has emerged as a central theme in recent research. Foundational works \cite{xu2018how, morris2019weisfeiler, morris2020weisfeiler} established that message-passing GNNs are at most as powerful as the 1-Weisfeiler-Lehman (1-WL) test in distinguishing non-isomorphic graphs. To overcome this limitation, higher-order GNNs inspired by the $k$-WL hierarchy—such as $k$-GNNs \cite{morris2019weisfeiler}, $k$-IGNs, and PPGNs \cite{maron2019provably} were proposed, though they often suffer from prohibitive computational costs. While the WL hierarchy offers a theoretical benchmark, it has been criticized for being overly complex (particularly for $k \geq 3$) and too coarse to distinguish between models bounded between successive WL levels \cite{morris2022speqnets}. More recently, Subgraph GNNs have gained traction as a practical and expressive alternative, with recent work \cite{qian2022ordered,frasca2022understanding, zhang2023complete} providing a refined understanding of their expressivity and connections to WL tests.  Recent work by \cite{zhang2024expressive} proposes a unifying framework, EPNN, for spectral-invariant GNNs and establishes a fine-grained expressiveness hierarchy, showing that all such models are strictly bounded by 3-WL. Complementing this, \cite{gai2025homomorphism} introduce a homomorphism-based expressivity theory for spectral GNNs, revealing that their expressive power corresponds to the class of parallel trees and lies strictly between 1-WL and 2-WL. This work answers key open questions and quantifies the expressiveness and substructure counting capabilities of spectral models. An extensive survey on the state of the art results on the expressive power of GNNs can be found in \cite{zhang2024expressive2}. 
 
\subsection{Approximation Ability of GNNs}
 Let a GNN be modeled as a function \( f_\theta: \mathcal{H} \to \mathbb{R}^{d \times n \times m} \), where \( \mathcal{G} \) denotes the space of graphs \( G = (A, X) \), with \( X \in \mathbb{R}^{d \times n} \) representing the node features, \( d \) being the feature dimension, \( n = |V| \) the number of nodes, and \( m = |\mathcal{H}| \) the number of graphs. The output of the GNN is denoted by \( f_\theta(G) \). The function approximation problem can then be stated as follows: given a target function \( f^* \) over the domain \( \mathcal{G} \), the goal is to approximate \( f^* \) using \( f_\theta \) within a desired level of accuracy. Models with greater expressive power are capable of learning more complex mapping functions. Neural networks are well-known for their expressive capacity, most notably characterized by the universal approximation theorem~\cite{cybenko1989approximation, hornik1991approximation}. As a subclass of neural networks, Graph Neural Networks (GNNs) also possess the ability to approximate functions, which serves as a basis for describing their expressive power. However, due to the inductive bias intrinsic to GNNs—particularly permutation invariance—the classical universal approximation results for standard neural networks cannot be directly applied to GNNs~\cite{battaglia2018relational}. To address this, \cite{maron2019universality,azizian2020expressive,keriven2019universal} analyze the kinds of functions that GNNs can approximate, focusing specifically on permutation-invariant graph functions. They formalize the set of functions a GNN can approximate as a measure of its expressive power.

\subsection{Fairness Aware Graph Learning.}
Fairness-aware graph learning has emerged as a critical area of research that aims to mitigate biases in predictions made over graph-structured data by accounting for sensitive attributes such as gender, race, or region~\cite{dai2022comprehensive, chen2024fairness, dong2023fairness, laclau2022survey}. A wide array of techniques has been developed to improve fairness by modifying different components of the graph learning pipeline, including graph structure, node features, and message passing schemes. Many of these methods primarily focus on single sensitive attributes, often binary or categorical, and hence do not address more complex scenarios involving multiple or overlapping group memberships. Regularization-based approaches such as FairGNN~\cite{dai2021say}, EDITS~\cite{dong2022edits}, and UGE~\cite{wang2022unbiased} aim to enforce group fairness constraints like demographic parity and equalized odds. Other methods like Fairwalk~\cite{rahman2019fairwalk}, DeBayes~\cite{buyl2020debayes}, and FairAdj~\cite{li2020dyadic} debias random walks, embeddings, or adjacency matrices. Algorithms such as NIFTY~\cite{agarwal2021towards} and FairVGNN~\cite{wang2022improving} tackle counterfactual fairness and information leakage, respectively, while message-passing variants like GMMD~\cite{zhu2023fairness} and FMP~\cite{jiang2022fmp} modify the GNN aggregation to produce fairer node embeddings. Recent advances also explore graph generation as a fairness mechanism, including FG-SMOTE~\cite{wang2025fg}, which performs sampling-based generation, and FairWire~\cite{kose2025fairwire}, which leverages diffusion models for fair graph construction. Despite the progress, most existing approaches are limited in scope: they assume a single known sensitive attribute, typically fail to account for intersectional or overlapping groups, and support only a narrow subset of fairness criteria. For instance, methods like Crosswalk~\cite{khajehnejad2021crosswalk} and FairEGM~\cite{current2022fairegm} are restricted to binary attributes, while EqGNN~\cite{singer2022eqgnn} specifically targets equalized odds. Moreover, some methods such as CFC~\cite{bose2019compositional} handle multiple attributes but are tailored to bipartite recommendation scenarios. Overall, fairness-aware graph learning methods are diverse in their strategies—from debiasing input features and adjacency matrices to modifying training objectives and model architectures—but substantial gaps remain in addressing overlapping group fairness and achieving general-purpose fairness guarantees.

\subsection{Fourier Analysis of Boolean Function.}
The analysis of Boolean functions offers a powerful framework to study functions \( f: \{-1,1\}^n \to \{-1,1\} \) using tools from harmonic analysis and probability~\cite{o2014analysis}. Any such function admits a unique representation as a multilinear polynomial:
\[
f(x) = \sum_{S \subseteq [n]} \widehat{f}(S) \chi_S(x), \quad \text{where } \chi_S(x) = \prod_{i \in S} x_i,
\]
and the coefficients \( \widehat{f}(S) \) are the Fourier coefficients with respect to the uniform measure. Parseval’s identity ensures that \( \sum_S \widehat{f}(S)^2 = 1 \) for Boolean-valued functions. The \emph{influence} of variable \( i \) is defined as \( \mathrm{Inf}_i(f) = \sum_{S \ni i} \widehat{f}(S)^2 \), and the total influence \( \mathrm{Inf}(f) \) captures the function’s average sensitivity to input perturbations. The \emph{noise stability} of \( f \) under correlation parameter \( \rho \in [0,1] \) is given by
\[
\mathrm{Stab}_\rho(f) = \sum_{S \subseteq [n]} \rho^{|S|} \widehat{f}(S)^2,
\]
which decays as high-degree coefficients dominate. The \emph{hypercontractive inequality} relates norms under noise and implies that functions with low-degree spectral concentration are stable under random perturbations~\cite{bonami1970etude, beckner1975inequalities}. Additionally, the \emph{junta theorem} states that if the total influence is small, then the function is close to depending only on a small subset of variables~\cite{friedgut1998boolean}. These tools are widely used in learning theory, social choice theory, and computational complexity, particularly in the study of threshold phenomena, noise sensitivity, and approximation by low-degree polynomials~\cite{o2014analysis, odonnell2008some}.

\section{Preliminaries}
Let $G = (V, E)$ be an attributed graph with $|V| = n$ nodes and $|E| = m$ edges. A collection of subpopulations $\cal G \subseteq 2^{V}$ is also given. A subpopulation $S \subseteq V$ is encoded as a binary string $\chi_S \in \{0,1\}^n$ where $\chi_S(i) = 1$ if $v_i \in S$ and $0$ otherwise. Thus, the collection $\cal G$ can be encoded as boolean function $f_{\cal G}: \{0,1\}^n \to \{0,1\}$ that indicates whether fairness should be enforced on $S$. In standard fairness applications $\cal G$ can be obtained directly from node attributes, but in our analysis, we consider $\cal G$ to be a ``rich'' set of possibly overlapping populations. Throughout this paper, we work with Boolean functions defined on the hypercube $\{-1,1\}^n$ rather than $\{0,1\}^n$. This convention is standard in the Fourier analysis of Boolean functions, as it yields a simpler algebraic structure. In particular, the orthonormal basis consists of parity functions $\chi_S(x) = \prod_{i \in S} x_i$ for $S \subseteq [n]$, and inner products and expectations are taken over the uniform measure on $\{-1,1\}^n$. 
Any Boolean function $f: \{0,1\}^n \to \{0,1\}$ can be transformed into an equivalent function $\tilde{f}: \{-1,1\}^n \to \{-1,1\}$ via a change of variables $x_i = 1 - 2z_i$, where $z_i \in \{0,1\}$ and $x_i \in \{-1,1\}$. This transformation preserves the structural properties of the function while enabling cleaner spectral representations.
A GNN architecture $\mathcal{A}$ with $T$ layers computes node embeddings $h_v^{(T)}$ via message passing. The output representation for each node depends on the local neighborhood up to $T$ hops.
We refer to standard circuit complexity classes used in Boolean function analysis \cite{arora2009computational}. The class AC$^0$ consists of constant-depth, polynomial-size circuits with unbounded fan-in AND, OR, and NOT gates. These circuits can compute simple functions like conjunctions, disjunctions, and short decision lists, but cannot express parity or majority. In contrast, NC$^1$ contains circuits of logarithmic depth and polynomial size with bounded fan-in gates, and can express a much broader class of functions, including parity, nested logic, and threshold-based conditions. These distinctions serve as the basis for our expressivity hierarchy in fairness-aware GNNs.

\section{Boolean Fairness Functions Expressivity of GNNs}
In this section, We extend the classical notion of GNN expressivity which is based on the WL test, to include the ability to distinguish graphs based on 
a given set of subpopulations $f_{\cal G}$. We first recall the meaning of the classical Weisfeilar-Lehman expressivity of GNNs.

\begin{definition}[1-WL Expressivity] 
A GNN architecture $\cal A$ is said to be \emph{1-WL expressive} if, for any two non-isomorphic graphs $G$ and $H$ on a set of $n$ nodes that are distinguishable by the 1-dimensional Weisfeiler-Lehman (1-WL) test, the GNN maps them to different node representations. Formally, for any such pair $(G,H)$, the multiset of node embeddings produced by the GNN differs:
\[
\text{1-WL}(G) \neq \text{1-WL}(H) \implies {\cal A}(G) \neq {\cal A}(H)
\]
where the $\neq$ sign refers to the multiset non equality (due to the permutation invariance of GNNs and WL test.).
\end{definition}
This means that the GNN is at least as powerful as the 1-WL test in distinguishing non-isomorphic graphs.

\begin{definition} [Function-level Expressivity] Let $G \neq H$ be two non-isomorphic graphs on $n$ nodes with associated fairness subpopulations $f_G$ and $f_{H}$. A GNN architecture $\mathcal{A}$ has function-level expressivity if:
\[
    f_G \not\equiv f_{H} \Rightarrow \mathcal{A}(G) \not\equiv \mathcal{A}(H)
\]
where the condition $f_G \not\equiv f_{H}$ refers to the isomorphism equivalence of the boolean functions $f_G$ and $f_H$. 
\end{definition}
The above definition reflects the requirement that semantically different subpopulations across graphs should lead to distinguishable embeddings.

\begin{theorem}
The graph isomorphism problem ($\mathsf{GI}$) reduces in polynomial time to the subpopulation boolean function isomorphism problem ($\mathsf{SubIso}$).
\end{theorem}

\begin{proof}
Let $G = (V, E_G)$ and $H = (V, E_H)$ be two undirected graphs on $n$ vertices, with vertex set $V = \{1, 2, \dots, n\}$. Define Boolean functions $f_G, f_H : \{0,1\}^n \to \{0,1\}$ as follows. For each $x \in \{0,1\}^n$ with Hamming weight $\mathrm{wt}(x) = 2$, let the support of $x$ be the unordered pair $\{u,v\} \subseteq V$. Then define:
\[
f_G(x) = 
\begin{cases}
1 & \text{if } \{u,v\} \in E_G \\
0 & \text{otherwise}
\end{cases}, \quad
f_H(x) = 
\begin{cases}
1 & \text{if } \{u,v\} \in E_H \\
0 & \text{otherwise}
\end{cases}
\]
and for all $x$ with $\mathrm{wt}(x) \neq 2$, define $f_G(x) = f_H(x) = 0$. Let $G \cong H$, then there exists a permutation $\pi \in S_n$ such that $\{u,v\} \in E_G \iff \{\pi(u), \pi(v)\} \in E_H$. This implies that for any $x \in \{0,1\}^n$ with $\mathrm{wt}(x) = 2$ and $\mathrm{supp}(x) = \{u,v\}$, we have:
\[
f_G(x) = 1 \iff f_H(\pi(x)) = 1
\]
and the same holds for inputs where the value is zero. Thus $f_G(x) = f_H(\pi(x))$ for all $x \in \{0,1\}^n$, i.e., $f_G$ and $f_H$ are isomorphic. Conversely, if $f_G$ and $f_H$ are isomorphic under a permutation $\pi \in S_n$, then for every $\{u,v\} \in V$, the edge $\{u,v\} \in E_G$ if and only if $\{\pi(u), \pi(v)\} \in E_H$. Therefore, $\pi$ defines an isomorphism from $G$ to $H$. The construction of $f_G$ and $f_H$ from $G$ and $H$ involves enumerating each edge and constructing its weight-2 binary encoding, which can be done in polynomial time. Thus, the reduction from graph isomorphism to subpopulation Boolean function isomorphism is polynomial-time computable.
\end{proof}

It is easy to see that in general $\mathsf{SubIso}$ doesn't reduce to $\mathsf{GI}$ because if the $\mathsf{SubIso}$ boolean function instance corresponds to subpopulations of size more than 3, then they can't be captured using graphs. The Weisfeiler-Lehman (WL) test, though strictly weaker than $\mathsf{GI}$, is widely used to benchmark the expressiveness of GNNs due to its structural similarity with message passing. This alignment has made WL a natural and practical proxy despite its limitations. In contrast, the subpopulation isomorphism problem—which we show to be harder than GI—captures a richer class of boolean subpopulation distinctions. Just as WL serves as a lower bound, subpopulation isomorphism can serve as a more stringent benchmark for GNN expressiveness. From complexity-theoretic terminology, we can write $\mathsf{WL} \subsetneq \mathsf{GI} \subsetneq   \mathsf{SubIso}$. The following result shows that if the subpopulation boolean function has a high Fourier degree, then no GNN with bounded message passing depth and standard aggregation functions can determine $f$ even approximately. 

\begin{theorem}
Let $f: \{0,1\}^n \to \{0,1\}$ be a Boolean function with Fourier degree $\deg_{\mathcal{F}}(f) = d$, i.e., there exists a subset $S \subseteq [n]$ with $|S| = d$ such that the Fourier coefficient $\hat{f}(S) \neq 0$. Then, for any message-passing GNN architecture $\mathcal{A}_T$ of depth $T$ defined over a bounded-degree graph with maximum degree $C$, and using standard permutation-invariant aggregation functions (e.g., sum, mean, max), if $T < \log_C d$, then:
\[
\inf_{g \in \mathcal{A}_T} \mathbb{E}_{x \sim \{0,1\}^n} \left[ |f(x) - g(x)| \right] \geq \Omega(1).
\]
In other words, no such GNN can exactly represent or even approximate $f$ to arbitrary precision.
\end{theorem}
\begin{proof}
We model the GNN as a local computation model over a graph of $n$ nodes, each node $v$ initialized with a Boolean input $x_v \in \{0,1\}$. A GNN of depth $T$ updates node representations via:
\[
h_v^{(t)} = \phi^{(t)}\left(h_v^{(t-1)}, \bigoplus_{u \in \mathcal{N}(v)} \psi^{(t)}(h_u^{(t-1)}) \right), \quad h_v^{(0)} = x_v,
\]
where $\phi^{(t)}$, $\psi^{(t)}$ are learnable functions and $\bigoplus$ denotes a permutation-invariant aggregator (sum, mean, or max).
Since the depth is $T$ and the graph has degree at most $C$, the output at any node depends only on the input features within its $T$-hop neighborhood, which contains at most $C^T$ nodes. Thus, the function $g \in \mathcal{A}_T$ computed by the GNN depends on at most $C^T$ variables.
It is known from Boolean function analysis that if a function depends on at most $k$ variables, then its Fourier expansion contains only subsets $S$ with $|S| \leq k$. Therefore, $\deg_{\mathcal{F}}(g) \leq C^T$ for all $g \in \mathcal{A}_T$.

Now suppose $f$ has $\deg_{\mathcal{F}}(f) = d > C^T$. Then there exists a subset $S$ with $|S| = d$ such that $\hat{f}(S) \neq 0$, but $\hat{g}(S) = 0$ for all $g \in \mathcal{A}_T$. Using Parseval’s identity and orthogonality of the Fourier basis, we obtain:
\[
\mathbb{E}_{x}[(f(x) - g(x))^2] = \sum_{S \subseteq [n]} (\hat{f}(S) - \hat{g}(S))^2 \geq \sum_{|S| = d} \hat{f}(S)^2.
\]
If we define $\epsilon := \sum_{|S|=d} \hat{f}(S)^2 > 0$, then:
\[
\mathbb{E}_x[(f(x) - g(x))^2] \geq \epsilon \quad \Rightarrow \quad \mathbb{E}_x[|f(x) - g(x)|] \geq \sqrt{\epsilon}.
\]
This implies that no GNN of depth $T < \log_C d$ can approximate $f$ within error less than $\sqrt{\epsilon}$, completing the proof.
\end{proof}

The following result shows that the subpopulation boolean functions agree on all subsets where all nodes are within $T$ hops of each other, the GNN cannot distinguish the graphs and  embeddings are equal. 

\begin{theorem}
Let $G$ and $H$ be two graphs with corresponding subpopulation Boolean functions $f_G$ and $f_H$, defined over subsets of nodes. Let $\mathcal{A}_T$ be a message-passing GNN with depth $T$ and standard permutation-invariant aggregation (sum, mean, max). Then, the following holds:

\begin{enumerate}
    \item If for every subset $S$ of nodes in $G$ and $H$ such that the maximum pairwise shortest-path distance between nodes in $S$ is at most $T$, it holds that $f_G(S) = f_H(S)$,
    then the node embeddings produced by $\mathcal{A}_T$ are identical i.e. $\Phi_T(G) = \Phi_T(H)$.
    
    \item Conversely, if there exists a subset $S$ where the maximum pairwise shortest-path distance between nodes in $S$ is at most $T$ such that $f_G(S) \neq f_H(S)$,
    then $\mathcal{A}_T$ can produce distinguishable embeddings i.e.
    $\Phi_T(G) \neq \Phi_T(H)$.
    
    \item Any difference between $f_G$ and $f_H$ that arises solely from subsets where the pairwise shortest-path distance exceeds $T$ cannot be captured by $\mathcal{A}_T$. In that case, 
    $\Phi_T(G) = \Phi_T(H)$.
    
\end{enumerate}

\end{theorem}

\begin{proof}
Message-passing GNNs with depth $T$ compute node embeddings by recursively aggregating information from each node’s $T$-hop neighborhood. Specifically, the embedding of node $v$ depends only on the multiset of node features and structural information within its $T$-hop neighborhood. 
\noindent{\textbf{Part 1.}} Assume that for every subset $S$ where the maximum pairwise shortest-path distance between nodes in $S$ is at most $T$, it holds that $f_G(S) = f_H(S)$.  This implies that all structural and feature patterns that are visible within any $T$-hop neighborhood are identical in $G$ and $H$. Consequently, each node in $G$ and the corresponding node in $H$ will aggregate identical messages during each of the $T$ layers. Therefore, the final node embeddings produced by $\mathcal{A}_T$ are identical i.e. $\Phi_T(G) = \Phi_T(H)$.\\
\noindent{\textbf{Part 2.}} Suppose there exists a subset $S$ such that the maximum pairwise shortest-path distance among nodes in $S$ is at most $T$ and $f_G(S) \neq f_H(S)$. This difference implies that there is some local structural or feature pattern that differs between $G$ and $H$ within the receptive field of the GNN. Since the GNN aggregates over such neighborhoods, this difference will propagate into the computation of node embeddings, leading to i.e.  $\Phi_T(G) \neq \Phi_T(H)$.\\
\noindent{\textbf{Part 3.}} If all differences between $f_G$ and $f_H$ occur only on subsets where the pairwise distance exceeds $T$, then no node in $G$ or $H$ has access to that information within its $T$-hop neighborhood. The GNN cannot access, aggregate, or represent such information. Thus, the node embeddings will be identical ensuring $\Phi_T(G) = \Phi_T(H)$.
\end{proof}

\section{Formal Expressivity Framework}
In this section, we provide a formalism to quantify the techniques which have been used to expressive power of GNNs.  

\begin{definition}[Expressivity Invariant]
Let $\mathcal{G}_n$ denote the set of graphs with $n$ nodes. An expressivity framework is defined by an invariant function
$I : \mathcal{G}_n \to \mathcal{S}$
where $\mathcal{S}$ is the summary space. Two graphs $G, H \in \mathcal{G}_n$ are equivalent under $I$ if
$I(G) = I(H)$. We consider the following invariants wherein the last invariant is proposed by us:

\begin{enumerate}
    \item \textbf{WL-$k$ Invariant:} This invariant is the most prevalanet one in GNN expressivility literature which is based on the equality of stable coloring histograms of $k$-tuples. 
    \[
    I_{\mathrm{WL}-k}(G) = \mathrm{ColorHist}_{k}(G)
    \]

    \item \textbf{Biconnectivity Invariant:} Introduced in \cite{zhang2023rethinking} this invariant is based on the equality of combinatorial properties of biconnectivity.  
    \[
    I_{\mathrm{BIC}}(G) = \left(\mathrm{AP}(G), \mathrm{Bridges}(G), \mathrm{BiconnComp}(G)\right)
    \]
    where AP($G$) denotes the set of articulation points of $G$.

    \item \textbf{Positional Encoding Invariant:} This invariant is based on the eigenvalues of the Graph Laplacian $L_{G}$
    \[
    I_{\mathrm{PE}}(G) = \mathrm{EigenSpectrum}(L_G)
    \]

    \item \textbf{Homomorphism Invariant:} This invariant proposed in \cite{zhang2024beyond} is based on the homomorphism counts of of a family of graphs $\cal F$.
    \[
    I_{\mathrm{HOM}}(G) = \left\{ \mbox{Hom}(F,G) : F \in {\cal F} \right\}
    \]

    \item \textbf{Subpopulation Boolean Function Invariant:} This is our proposed invariant which is a boolean function $f_G$ defined on the node set $V$ of $G$.
    \[
    I_{\mathrm{SBI}}(G) = \left\{ f_G(S) : S \subseteq V \right\}
    \]

    \item \textbf{Graph Isomorphism Invariant:} This invariant maps a graph $G$ to its adjacency-preserving canonical form (or any canonical labeling under isomorphism). 
    \[
    I_{\mathrm{GI}}(G) = \mbox{orbit}(G)
    \]
    where $\mbox{orbit}(G)$ is the isomorphism class of $G$ under adjacency-preserving bijections of nodes.
    
\end{enumerate}
\end{definition}

\begin{theorem}
For any graph invariant $I$, let $\mathrm{Im}(I)$ to be the image space of $I$ over the set of all graphs: $\mathrm{Im}(I) := \{ I(G) \mid G \text{ is a graph} \}$. This represents the set of all possible invariant values under $I$. Two graphs $G$ and $H$ are indistinguishable under $I$ if and only if $I(G) = I(H)$. Then, the following strict hierarchy holds:
\[
\mathrm{Im}(I_{\mathrm{WL}-k}) 
\subsetneq \mathrm{Im}(I_{\mathrm{HOM}}) 
\subsetneq \mathrm{Im}(I_{\mathrm{GI}}) 
\subsetneq \mathrm{Im}(I_{\mathrm{SBI}})
\]
and
\[
\mathrm{Im}(I_{\mathrm{BIC}}) \subsetneq \mathrm{Im}(I_{\mathrm{SBI}}).
\]
\end{theorem}
\begin{proof}
We prove each strict inclusion in the hierarchy.

\textbf{1.} $\mathrm{Im}(I_{\mathrm{WL}-k}) \subsetneq \mathrm{Im}(I_{\mathrm{HOM}})$: The $k$-WL invariant is equivalent to the homomorphism counts of all graphs of treewidth at most $k$ (Dell et al., 2018; Grohe, 2017). The homomorphism expressivity framework \cite{zhang2024beyond} generalizes this by including graphs characterized by Nested Ear Decompositions (NED) or their variants (e.g., strong NED, almost-strong NED) depending on the GNN class. For example, Local 2-GNN can distinguish cycles of length up to 7, while 2-WL cannot distinguish them (this is formally shown in their Theorem 3.4). Therefore, there exist graphs $G$ and $H$ such that:
\[
I_{\mathrm{WL}-k}(G) = I_{\mathrm{WL}-k}(H) \text{ but } I_{\mathrm{HOM}}(G) \neq I_{\mathrm{HOM}}(H).
\]
This proves the inclusion is strict.

\textbf{2.} $\mathrm{Im}(I_{\mathrm{HOM}}) \subsetneq \mathrm{Im}(I_{\mathrm{GI}})$:

The homomorphism invariant depends on counts of patterns from a fixed family $\mathcal{F}_M$.
It is well-known in graph theory that distinct non-isomorphic graphs can be homomorphically equivalent with respect to certain subgraph counts.
For instance, two graphs may have identical homomorphism counts for all patterns in $\mathcal{F}_M$ but differ in adjacency structures that are not captured by those patterns.
Graph isomorphism, on the other hand, checks for the existence of an adjacency-preserving bijection, which fully determines the graph structure. Therefore:
\[
I_{\mathrm{HOM}}(G) = I_{\mathrm{HOM}}(H) \rlap{\(\quad\not\)}\implies G \cong H,
\]
but
\[
I_{\mathrm{GI}}(G) = I_{\mathrm{GI}}(H) \iff G \cong H.
\]
Hence, $\mathrm{Im}(I_{\mathrm{HOM}}) \subsetneq \mathrm{Im}(I_{\mathrm{GI}})$.

\textbf{3.} $\mathrm{Im}(I_{\mathrm{GI}}) \subsetneq \mathrm{Im}(I_{\mathrm{SBI}})$:

Graph isomorphism checks for adjacency-preserving bijections but does not capture Boolean constraints over arbitrary subpopulations of nodes. The subpopulation Boolean function invariant $I_{\mathrm{SBI}}$ encodes the entire Boolean function over all node subsets, including properties unrelated to adjacency. For example, $I_{\mathrm{SBI}}$ can encode parity functions such as "Is the number of triangles in the graph even?", which $I_{\mathrm{GI}}$ cannot distinguish if the graphs are otherwise isomorphic in edge structure but differ in higher-order or attribute-based Boolean properties.
- Therefore, the image space of $I_{\mathrm{SBI}}$ is strictly finer than that of $I_{\mathrm{GI}}$.

\textbf{4.} $\mathrm{Im}(I_{\mathrm{BIC}}) \subsetneq \mathrm{Im}(I_{\mathrm{SBI}})$:

The biconnectivity invariant encodes whether two graphs have the same decomposition into biconnected components.
This is a coarse structural property; many non-isomorphic graphs share the same biconnectivity decomposition.
Since the subpopulation Boolean function can encode not only biconnectivity but arbitrary logical constraints over node subsets, it strictly subsumes the biconnectivity invariant.
Hence the inclusion is strict.

Combining all steps, the complete strict hierarchy is established.
\end{proof}

\begin{theorem}
Let $I$ be any invariant respecting graph isomorphism. Then
\[
\exists F \text{ such that } I(G) = F(I_{\mathrm{SBI}}(G))
\]
i.e., SBI is a maximal combinatorial invariant.
\end{theorem}

\begin{proof}
A graph is fully characterized by its induced subgraph set. The Boolean function $f_G$ encodes this completely by mapping $S$ to whether the induced subgraph satisfies a property. Therefore, any isomorphism-invariant function is expressible as a function of $I_{\mathrm{SBI}}$.
\end{proof}

\subsection{Expressivity in terms of Efficiently Computable Boolean Functions} 
While the Subpopulation Boolean Isomorphism (SBI) framework offers maximal expressivity beyond classical invariants, it introduces computational challenges. Unlike efficiently computable invariants like WL labels or homomorphism counts, the SBI invariant is hard to compute and compare when defined over arbitrary Boolean functions \cite{bohler2004complexity}. However, fairness-aware applications typically involve simple, interpretable subpopulation functions—such as group membership or attribute conjunctions—which are efficiently computable. This motivates restricting SBI to Boolean functions in circuit classes like $\mathrm{AC}^0$ and $\mathrm{NC}^1$, enabling a principled trade-off between expressivity and tractability in fairness-driven graph learning.

\begin{definition}
Given a circuit class $\mathcal{C}$ (e.g., $\mathrm{AC}^0$, $\mathrm{NC}^1$), define the circuit-constrained subpopulation Boolean isomorphism (SBI) invariant as:
\[
I_{\mathrm{SBI}(\mathcal{C})}(G) := f_G
\]
where $f_G : 2^{V(G)} \to \{0,1\}$ is a Boolean function over node subsets that satisfies $f_G \in \mathcal{C}$. Two graphs $G$ and $H$ are indistinguishable under $I_{\mathrm{SBI}(\mathcal{C})}$ if and only if their corresponding Boolean functions are isomorphic up to a relabeling of the node set.
\end{definition}

\begin{theorem}
If the subpopulation Boolean function $f_G$ is restricted to encoding adjacency only, i.e.,
\[
f_G(S) = \begin{cases}
1 & \text{if } S = \{u,v\} \text{ and } (u,v) \in E(G) \\
0 & \text{otherwise}
\end{cases}
\]
then
\[
\mathrm{Im}(I_{\mathrm{GI}}) = \mathrm{Im}(I_{\mathrm{SBI}(\mathrm{Adj})})
\]
\end{theorem}

\begin{proof}
Both invariants capture the adjacency matrix up to isomorphism. The function $f_G$ records edge presence exactly. Two graphs have the same $f_G$ if and only if their adjacency matrices are identical modulo node relabeling, which is precisely the condition for graph isomorphism. Therefore, the image spaces are equal.
\end{proof}

\begin{theorem}
When the subpopulation Boolean function $f_G$ is restricted to the class $\mathrm{AC}^0$, the expressivity of the SBI invariant is identical to that of the graph isomorphism invariant:
\[
\mathrm{Im}(I_{\mathrm{SBI}(\mathrm{AC}^0)}) = \mathrm{Im}(I_{\mathrm{GI}})
\]
\end{theorem}

\begin{proof}
$\mathrm{AC}^0$ circuits are constant-depth, polynomial-size circuits with AND, OR, and NOT gates of unbounded fan-in. It is well-established that $\mathrm{AC}^0$ cannot compute global properties such as parity, triangle parity, or connectedness (Håstad, 1987). The only structural properties that an $\mathrm{AC}^0$-computable function can capture are local adjacency patterns and their constant-depth compositions. These properties are already captured fully by the adjacency matrix modulo isomorphism, which is what $I_{\mathrm{GI}}$ represents. Therefore, $I_{\mathrm{SBI}(\mathrm{AC}^0)}(G) = I_{\mathrm{SBI}(\mathrm{AC}^0)}(H)$ if and only if $I_{\mathrm{GI}}(G) = I_{\mathrm{GI}}(H)$. Thus, their image spaces are identical.
\end{proof}

\begin{theorem}
When the subpopulation Boolean function $f_G$ is restricted to $\mathrm{NC}^1$ (logarithmic-depth, polynomial-size circuits with bounded fan-in), the SBI invariant is strictly more expressive than the graph isomorphism invariant:
\[
\mathrm{Im}(I_{\mathrm{GI}}) \subsetneq \mathrm{Im}(I_{\mathrm{SBI}(\mathrm{NC}^1)})
\]
\end{theorem}

\subsection{Expressivity using Fourier Analysis of Boolean Functions}
While circuit complexity captures the computational feasibility of subpopulation functions, it overlooks their structural properties such as whether they rely on local or global interactions. Fourier analysis complements this by decomposing Boolean functions into components reflecting interactions over variable subsets, revealing whether a function encodes simple motifs or global patterns like parity. Measures like degree and influence offer insights orthogonal to circuit depth, making Fourier analysis especially useful for analyzing GNN limitations where aggregation and expressivity are tied to functional structure. This provides a sharper understanding of the function-theoretic limits of SBI and GNNs.

\begin{definition}
Given a Boolean function $f_G : \{0,1\}^n \to \{0,1\}$ representing the subpopulation structure of graph $G$, let $\deg_{\mathcal{F}}(f_G)$ denote its Fourier degree. The Fourier-constrained SBI invariant is defined as:
\[
I_{\mathrm{SBI}(\deg \leq d)}(G) := f_G \text{ such that } \deg_{\mathcal{F}}(f_G) \leq d
\]
i.e., the maximum size of any subset $S$ with nonzero Fourier coefficient $\hat{f}_G(S)$ is at most $d$.
\end{definition}

\begin{theorem}
$\mathrm{Im}(I_{\mathrm{SBI}(\deg \leq 2)}) = \mathrm{Im}(I_{\mathrm{GI}})$
\end{theorem}

\begin{proof}
Any Boolean function with Fourier degree at most 2 depends only on linear terms and pairwise interactions between input variables. In the subpopulation context, this means the function can encode:
\begin{itemize}
    \item Presence or absence of individual nodes (irrelevant for unlabeled graphs).
    \item Presence or absence of edges, i.e., pairs $(u,v)$ such that $(u,v) \in E(G)$.
\end{itemize}
Higher-order structures such as triangles, cliques, or global properties require Fourier degree $\geq 3$. The adjacency matrix is fully captured by pairwise terms. Therefore, the SBI invariant under degree-2 constraint is equivalent to the graph isomorphism invariant.
\end{proof}

\begin{definition}
Let $\mathrm{Inf}_i(f_G)$ denote the influence of variable $i$ on $f_G$. Define the total influence as:
\[
\mathrm{Inf}(f_G) = \sum_{i} \mathrm{Inf}_i(f_G)
\]
The influence-constrained SBI invariant is:
\[
I_{\mathrm{SBI}(\mathrm{Inf} \leq T)}(G) := f_G \text{ such that } \mathrm{Inf}(f_G) \leq T
\]
\end{definition}

\begin{theorem}
Low-influence SBI functions correspond to subpopulation properties that are stable under small perturbations of node sets. The expressivity of $I_{\mathrm{SBI}(\mathrm{Inf} \leq T)}$ excludes brittle, parity-like functions and emphasizes smooth, low-complexity subgroup interactions.
\end{theorem}

\begin{proof}
By the KKL theorem and subsequent results, Boolean functions with low total influence are noise-stable and tend to have Fourier mass on low-degree terms. Therefore, SBI constrained by influence behaves similarly to one constrained by Fourier concentration but captures a different smoothness property related to stability under input perturbations rather than strict degree bounds.
\end{proof}

\section{Application to Fair Graph Learning}
Current state-of-the-art  fair GNN approaches do not provide a general framework which can address their ability to ensure fairness in arbitrary subpopulations and can usually cater to simple sensitive attributes like `gender', `race', `region'  etc. rather than complex combinations of these attributes (e.g. ``young females from region $A$'') which are often useful in fairness applications and referrred to as \emph{intersectionality}. Based on our proposed framework in the previous Section we can assess the power of a given fair GNN architecture based on their ability to distinguish different subpopulation boolean functions. For example, we can make statements of the form: \emph{a fair GNN architecture that is 1-WL bounded can't ensure fairness for arbitrary subpopulations}. Thus, our framework provides a theoretical explanation for the failure modes of fairness algorithms that rely on limited expressivity models. Although our framework is naturally fits for node level fairness tasks like fair node classification, it can easily be adapted for edge level tasks like fair link prediction by considering boolean functions of the form $f: 2^{\binom{n}{2}} \rightarrow \{0,1\}$. 
Overall, the SBI-based framework provides a principled and extensible foundation for designing graph learning models that are both expressive and fairness-aware, particularly in settings where group fairness must account for complex population structures inherent in real-world networks. Using our framework we can prove an upper bound on the complexity of subpopulations that can be handled by a particular fairness aware graph learning.
\begin{table*}[ht]
\centering
\begin{tabular}{|l|c|c|c|}
\hline
\textbf{Method} & \textbf{Mechanism Type} & \textbf{Max Boolean Degree} & \textbf{Circuit Class} \\
\hline
Fairwalk     & Walk Sampling   & 1   & AC$^0$ \\
FairAdj      & Edge Perturbation & 1   & AC$^0$ \\
EDITS        & Embedding Debiasing & 1   & AC$^0$ \\
UGE          & Attribute-level Generator & 2 & AC$^0$–NC$^0$ \\
FairWire     & Diffusion-based with regularizer & 2 & Approx. NC$^1$ \\
FairSIN      & Invariant Penalty + Embedding Alignment & 3 & NC$^1$ \\
{\bf FairSBF} & Circuit-aware Fairness Optimization & $\geq 4$ & \textbf{NC$^1$+} \\
\hline
\end{tabular}
\label{tab:fairgnn-expressivity}
\caption{Expressive capacity of existing Fair GNNs in terms of the complexity of subpopulation Boolean functions they can handle. Complexity is described using circuit classes and Boolean properties (e.g., degree, influence). Our method \textbf{FairSBF} is the only method capable of addressing general Boolean functions in the class $\mathrm{NC^1}$, unlike prior approaches that are limited to shallow conjunctions or walk-based subpopulations.}
\end{table*}

\begin{figure*}[ht]
\centering
\includegraphics[scale=.6]{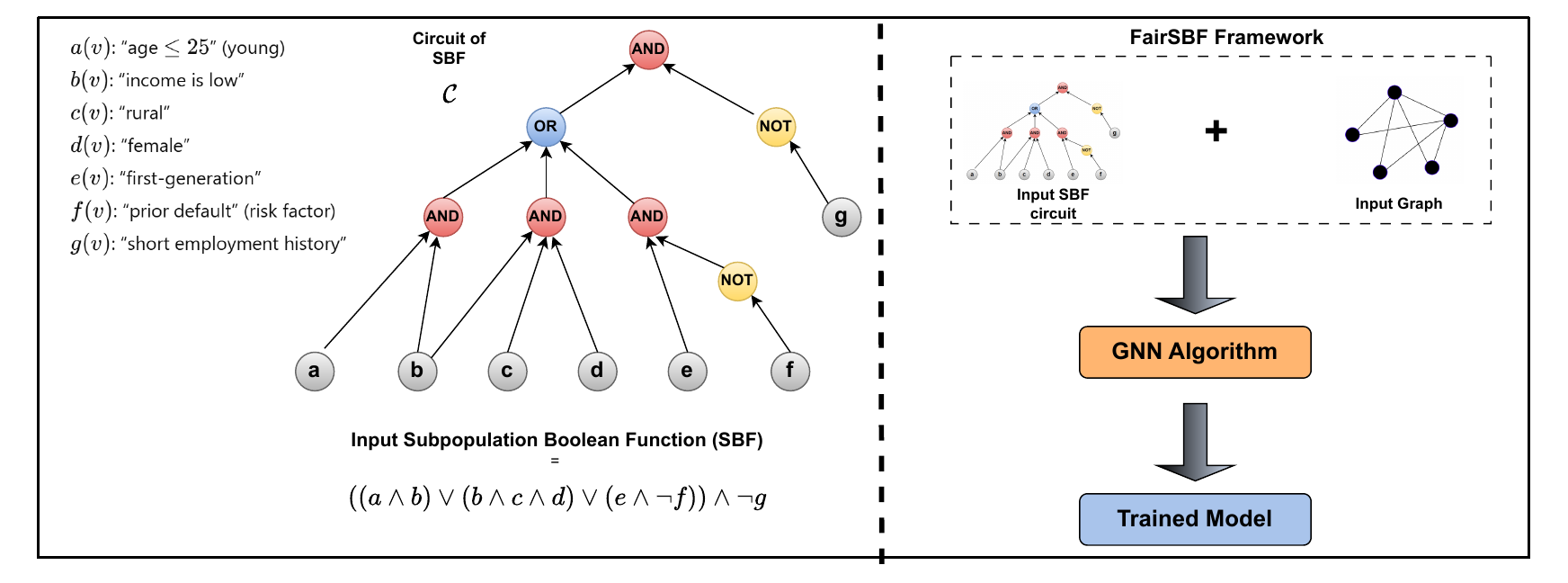}
\caption{The left part of the figure shows an example of an input Subpopulation Boolean Function along with the corresponding circuit involving several attributes. The right part shows the FairSBF framework which takes as input the (training) graph as well as the SBF circuit to a GNN algorithm, and creates a new regularizer based training algorithm to produce the final trained model.
}
\label{fig:sbfcircuit}
\end{figure*}

\begin{figure}[ht]
\centering
\includegraphics[width=0.48\textwidth]{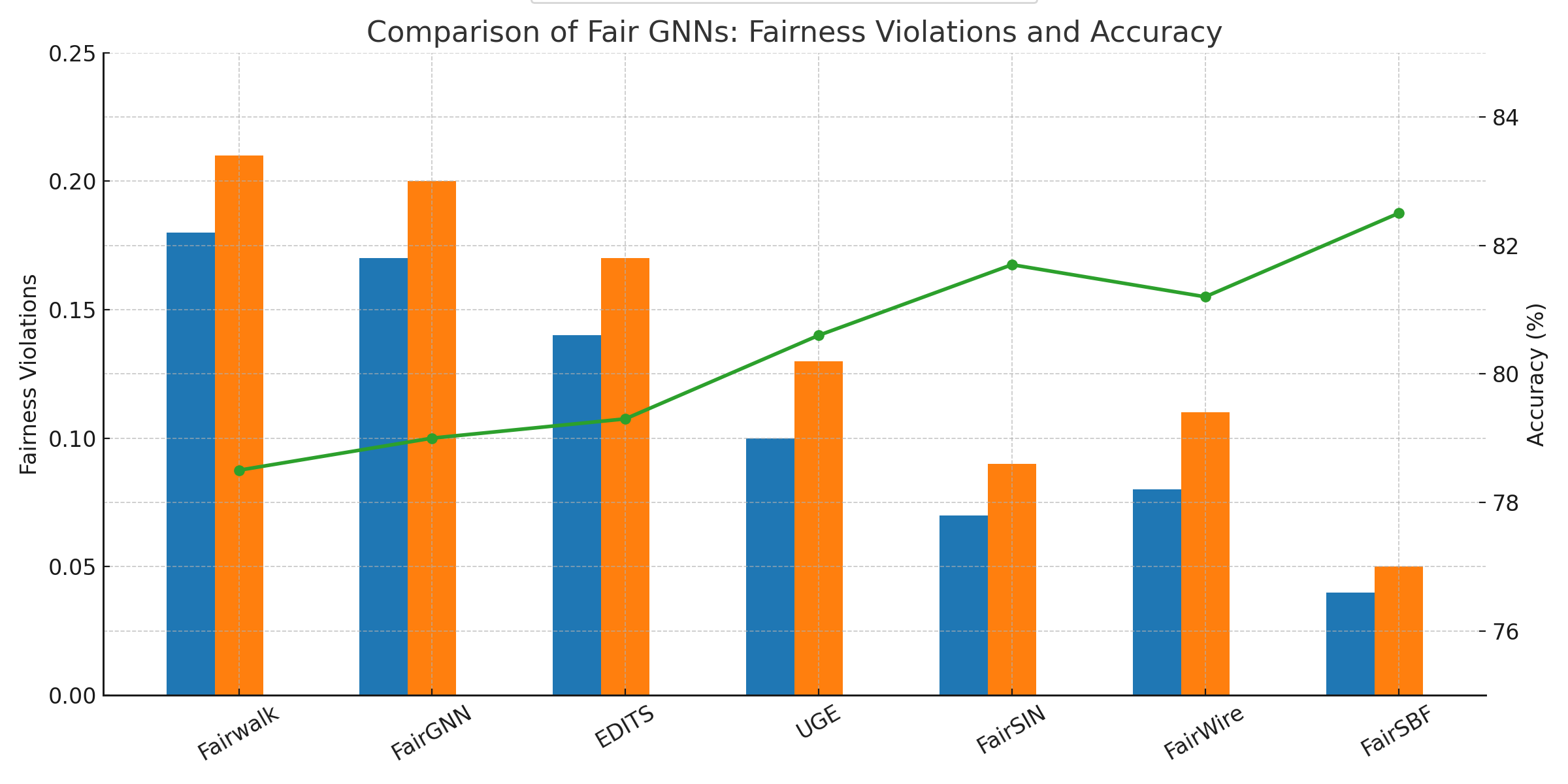}
\caption{Performance comparison of Fair GNNs on the Pokec-z dataset with complex subpopulations. Our method \textbf{FairSBF} achieves the best trade-off, showing the highest accuracy and the lowest fairness violations (DDP, DEO). Unlike other methods which degrade on high-complexity Boolean functions, FairSBF remains robust due to its circuit-theoretic fairness layer.
}
\label{fig:faircircuit-results}
\end{figure}

\begin{theorem}
Let \( f: \{0,1\}^n \to \{0,1\} \) be a Boolean function representing a subpopulation fairness constraint. Suppose a GNN model \( \mathcal{A} \) satisfies one of the following properties:

\begin{itemize}
    \item \textbf{Property \( \mathcal{P}_{\text{adv}} \)} (Adversarial Feature Alignment): The model enforces fairness by aligning latent representations across groups defined by observed node-wise sensitive attributes using an adversarial discriminator.

    \item \textbf{Property \( \mathcal{P}_{\text{edge}} \)} (Feature-based Edge Masking): The model enforces fairness by reweighting or masking the adjacency matrix using conditions based on pairwise sensitive attribute similarity.

    \item \textbf{Property \( \mathcal{P}_{\text{walk}} \)} (Feature-based Biased Walks): The model enforces fairness by modifying random walks to prefer transitions between nodes sharing sensitive attribute values.
\end{itemize}

Then, for any GNN \( \mathcal{A} \) satisfying one of the above properties, the model can enforce fairness only for subpopulations defined by Boolean functions \( f \) satisfying the condition that $\deg_{\mathcal{F}}(f)$ and $f$ depends only on observed sensitive attributes (node-wise or pairwise). In particular, \( \mathcal{A} \) cannot express fairness over subpopulations defined by:
\begin{itemize}
    \item conjunctions of 3 or more attributes (e.g., \( f(x) = x_1 \land x_2 \land x_3 \)),
    \item parity or XOR-like functions (e.g., \( f(x) = x_1 \oplus x_2 \oplus x_3 \)),
    \item structural subgroups (e.g., nodes with common neighbors or global motifs),
    \item non-local or multi-hop interactions between sensitive attribute patterns.
\end{itemize}
\end{theorem}

In the following we provide an algorithm {\bf FairSBF} that can cater to the aforementioned boolean functions. FairSBF is trained iteratively, alternating between standard GNN message passing and fairness-aware supervision. In each epoch, node embeddings are updated via GNN layers, and predictions are evaluated for fairness across all subpopulations defined by the intermediate gates of a Boolean circuit $\mathcal{C}_f$. The fairness loss, aggregated over these gates, is combined with task loss to guide parameter updates. This iterative approach arises from the need to handle fairness constraints that are inherently non-decomposable and depend on internal logical structure. By embedding these constraints within each training loop, FairSBF incrementally aligns predictions with fairness across logically structured subpopulations.

\section{Experiments}
We evaluate our proposed algorithm \textbf{FairSBF} on the Pokec-z dataset, a real-world social network with three sensitive attributes namely \texttt{gender}, \texttt{age}, and \texttt{region} of which \texttt{gender} and \texttt{region} are binary and \texttt{age} is multivalent. In our experiments on the Pokec-z dataset, we construct subpopulations using the following Boolean function over gender (0/1), region (0/1), and age (0/1) : $f(x) = (({\tt gender} = 0 \wedge {\tt region} = 0) \vee ({\tt gender} = 1 \wedge {\tt age} = 0)) \oplus ({\tt region} = 1 \wedge {\tt age} = 1)$ where $\oplus$ represents the XOR boolean function. This function captures rich overlapping subgroups and belongs to the NC$^1$ circuit class due to the XOR function. For prior baselines such as FairGNN, UGE, and FairSIN, which cannot handle such expressive subpopulation logic, we use simplified AC$^0$ variants of the function (e.g., conjunctions over one or two attributes). As shown in Figure~\ref{fig:faircircuit-results}, our proposed FairSBF model achieves significantly lower fairness violations under the full NC$^1$ function, while other methods fail to generalize to this richer setting.

\section{Conclusion}
We introduced a Boolean function-theoretic framework for analyzing GNN expressivity in fairness-aware learning, using subpopulation Boolean functions to capture rich, structured fairness constraints. This approach allowed us to formally characterize the limitations of existing fair GNNs, which are typically restricted to shallow, attribute-level subgroups. Our proposed method, \textbf{FairSBF}, addresses these expressivity gaps by enforcing fairness over complex subpopulations defined by Boolean circuits in AC$^0$ and NC$^1$. By leveraging the internal structure of these circuits during training, FairSBF achieves fairness not only on the final subgroup but also across intermediate logical components. Empirical results on real-world graphs validate the theoretical advantages, showing that FairSBF consistently outperforms prior baselines in both fairness violation and predictive accuracy.

\begin{algorithm}[tb]

\SetAlgoLined

\caption{{\bf FairSBF}- Fairness-Aware GNN with Subpopulation Boolean Function $f \in \text{AC}^0 \cup \text{NC}^1$}
\KwIn{Graph $G = (V, E)$, node features $X \in \mathbb{R}^{|V| \times d}$, initial embeddings $H^{(0)}$, Boolean function $f$ with circuit $\mathcal{C}_f$, classifier $\hat{y}_v$}
\KwOut{Trained GNN with fairness-aware parameters}

\For{each training epoch}{
    Compute node embeddings $H^{(T)}$ using GNN\;
    Predict node outputs: $\hat{y}_v = \text{Classifier}(h_v^{(T)})$\;
    Initialize $\mathcal{L}_{\text{fair}} \gets 0$\;

    \For{each intermediate gate $g$ in $\mathcal{C}_f$}{
        Define subpopulation $S_g \gets \{ v \in V \mid g(x_v) = 1 \}$\;
        \If{$0 < |S_g| < |V|$}{
            Compute baseline mean: $\mu \gets \frac{1}{|V|} \sum_{v \in V} \hat{y}_v$\;
            Compute fairness gap: $\Delta_g \gets \left| \frac{1}{|S_g|} \sum_{v \in S_g} \hat{y}_v - \mu \right|$\;
            Update fairness loss: $\mathcal{L}_{\text{fair}} \gets \mathcal{L}_{\text{fair}} + \Delta_g$\;
        }
    }

    Compute task loss: $\mathcal{L}_{\text{task}} \gets \text{CrossEntropy}(\hat{y}, y)$\;
    Compute total loss: $\mathcal{L} \gets \mathcal{L}_{\text{task}} + \lambda \cdot \mathcal{L}_{\text{fair}}$\;
    Backpropagate and update model parameters using $\mathcal{L}$\;
}
\Return{Trained GNN model}
\end{algorithm}

\ifCLASSOPTIONcaptionsoff
  \newpage
\fi



%
\bibliographystyle{IEEEtran}
\bibliography{sample-base}

\begin{IEEEbiography}
    [{\includegraphics[width=1.05in,height=1.25in]{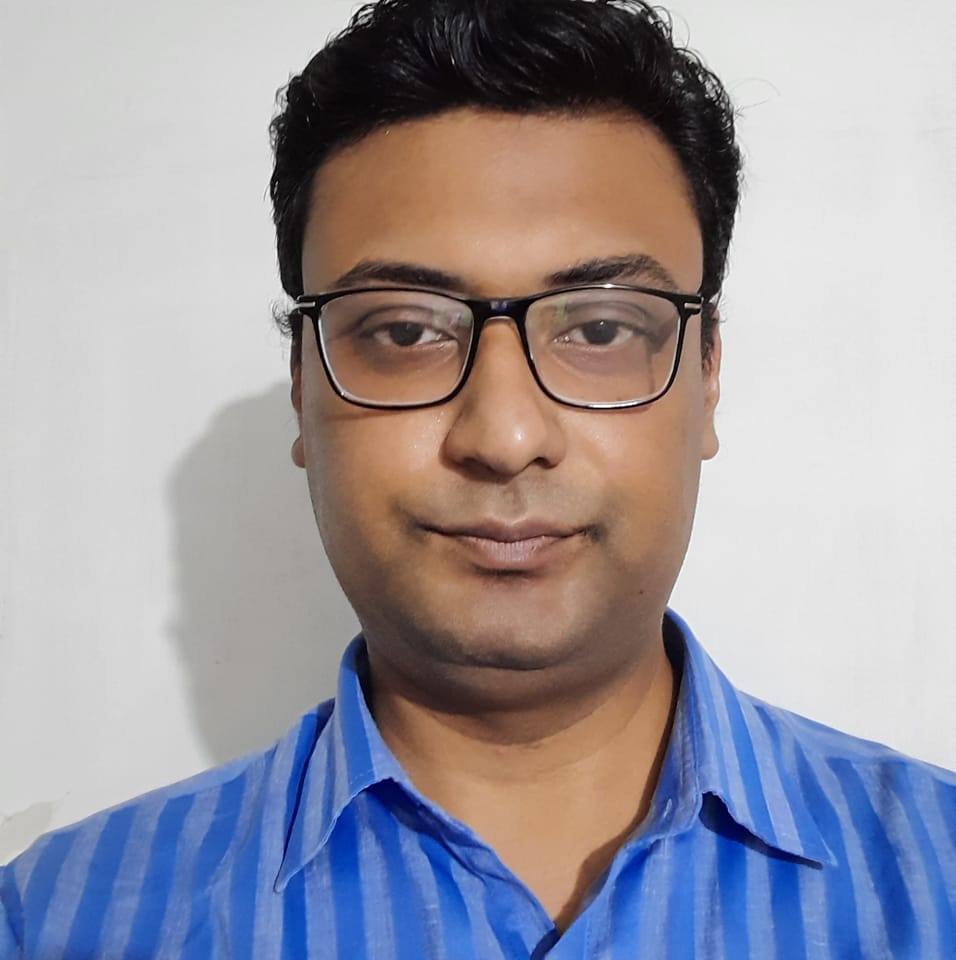}}]{Manjish Pal}
received his B-Tech and M-Tech degrees from IIT-Kanpur and PhD degree from IIT-Kharagpur, India. He has visited several world renowned universities like Princeton University, ETH Zurich, Tel-Aviv University and was an Assistant Professor of Computer Science at NIT-Meghalaya, India for five years. He is currently a Research Fellow in National Law University Meghalaya at Shillong, India. His interests lie in Fairness in Machine Learning, Optimization and Combinatorics.
\end{IEEEbiography}





\end{document}